\documentclass{article}

\usepackage{arxiv}

\usepackage[utf8]{inputenc} 
\usepackage[T1]{fontenc}    
\usepackage{hyperref}       
\usepackage{url}            
\usepackage{booktabs}       
\usepackage{amsfonts}       
\usepackage{nicefrac}       
\usepackage{microtype}      
\usepackage{lipsum}
\usepackage{verbatim}
\usepackage{array}
\usepackage{multirow}
\usepackage{graphicx}
\usepackage{amssymb,amsmath}
\usepackage{pgf}
\usepackage{tikz}
\usepackage{bbm}
\usepackage{hyperref}
\usepackage[numbers]{natbib}
\usepackage{subcaption}
\usepackage{amsthm}

\newcommand{\tuple}[2]{\langle #1, \, #2 \rangle}
\newcommand{\Ch}{C_{H}}
\newcommand{\Chu}{C_{Human}}
\newcommand{\Cr}{C_{R}}
\newcommand{\Cro}{C_{Robot}}

\newtheorem{theorem}{Theorem}

\newboolean{include-notes}
\setboolean{include-notes}{true}

\newcommand{\TODO}[1]{\ifthenelse{\boolean{include-notes}}
 {{\color{red} TODO: #1}}{}}
\newcommand{\Rohin}[1]{\ifthenelse{\boolean{include-notes}}
 {{\color{magenta} RS: #1}}{}}
\newcommand{\adnote}[1]{\ifthenelse{\boolean{include-notes}}
 {{\color{blue}AD: #1}}{}}
\definecolor{goodblue}{rgb}{0.36, 0.54, 0.66}
\newcommand{\Rach}[1]{\ifthenelse{\boolean{include-notes}}
 {{\color{goodblue} RF: #1}}{}}
 
 \newcommand\blfootnote[1]{%
  \begingroup
  \renewcommand\thefootnote{}\footnote{#1}%
  \addtocounter{footnote}{-1}%
  \endgroup
}

\usepackage{latexsym}

\title{Choice Set Misspecification in Reward Inference}

\author{
    Rachel Freedman\thanks{\texttt{rachel.freedman@berkeley.edu}}\\
    University of California, Berkeley\\
\And
    Rohin Shah\\
    DeepMind
\And
    Anca Dragan\\
    University of California, Berkeley
}

\graphicspath{ {fig/} }
\begin{document}

\maketitle

\begin{abstract}
  Specifying reward functions for robots that operate in environments without a natural reward signal can be challenging, and incorrectly specified rewards can incentivise degenerate or dangerous behavior. A promising alternative to manually specifying reward functions is to enable robots to infer them from human feedback, like demonstrations or corrections. To interpret this feedback, robots treat as approximately optimal a choice the person makes from a \emph{choice set}, like the set of possible trajectories they could have demonstrated or possible corrections they could have made. In this work, we introduce the idea that the choice set itself might be difficult to specify, and analyze \emph{choice set misspecification}: what happens as the robot makes incorrect assumptions about the set of choices from which the human selects their feedback. We propose a classification of different kinds of choice set misspecification, and show that these different classes lead to meaningful differences in the inferred reward and resulting performance. While we would normally expect misspecification to hurt, we find that certain kinds of misspecification are neither helpful nor harmful (in expectation). However, in other situations, misspecification can be extremely harmful, leading the robot to believe the \emph{opposite} of what it should believe. We hope our results will allow for better prediction and response to the effects of misspecification in real-world reward inference.
\end{abstract}

\section{Introduction}

    \blfootnote{Copyright © 2020 for this paper by its authors. Use permitted under Creative Commons License Attribution 4.0 International (CC BY 4.0).}Specifying reward functions for robots that operate in environments without a natural reward signal can be challenging, and incorrectly specified rewards can incentivise degenerate or dangerous behavior~\cite{Leike2018,Krakovna2018}. A promising alternative to manually specifying reward functions is to design techniques that allow robots to infer them from observing and interacting with humans.
    
    \begin{figure}
        \centering
        \includegraphics[width=.6\linewidth]{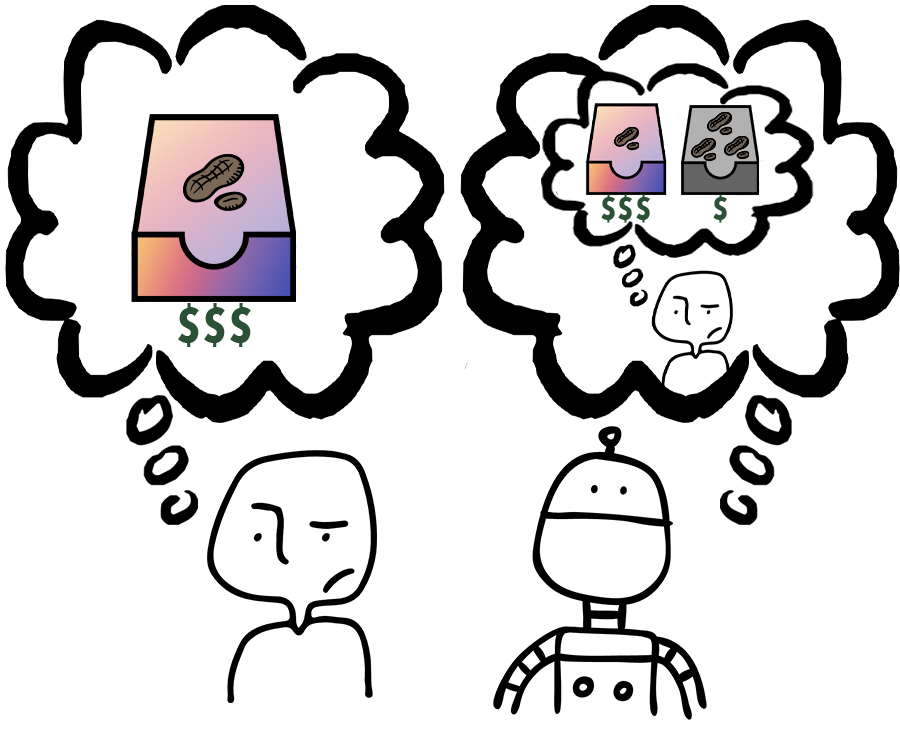} 
        \caption{Example \textit{choice set misspecification}: The human chooses a pack of peanuts at the supermarket. They only notice the expensive one because it has flashy packaging, so that's the one they buy. However, the robot incorrectly assumes that the human can see both the expensive flashy one and the cheap one with dull packaging but extra peanuts. As a result, the robot incorrectly infers that the human likes flashy packaging, paying more, and getting fewer peanuts.}
        \label{fig:front_fig}
    \end{figure}
    
    These techniques typically model humans as optimal or noisily optimal. Unfortunately, humans tend to deviate from optimality in systematically biased ways~\cite{Kahneman29179,Choi2014}. Recent work improves upon these models by modeling pedagogy~\cite{hadfield2016cooperative}, strategic behavior~\cite{waugh2013computational}, risk aversion~\cite{majumdar2017}, hyperbolic discounting~\cite{Evans2015}, or indifference between similar options~\cite{bobu2020less}. However, given the complexity of human behavior, our human models will likely always be at least somewhat misspecified~\cite{Steinhardt2017}.
    
    One way to formally characterize misspecification is as a misalignment between the real human and the robot's assumptions about the human. Recent work in this vein has examined incorrect assumptions about the human's hypothesis space of rewards~\cite{bobu2020quantifying}, their dynamics model of the world~\cite{reddy2018you}, and their level of pedagogic behavior~\cite{milli2019literal}. In this work, we identify another potential source of misalignment: \emph{what if the robot is wrong about what feedback the human could have given?} Consider the situation illustrated in Figure~\ref{fig:front_fig}, in which the robot observes the human going grocery shopping. While the grocery store contains two packages of peanuts, the human only notices the more expensive version with flashy packaging, and so buys that one. If the robot doesn't realize that the human was effectively unable to evaluate the cheaper package on its merits, it will learn that the human values flashy packaging.
    
    We formalize this in the recent framework of reward-rational implicit choice (RRiC)~\cite{jeon2020reward} as misspecification in the human \emph{choice set}, which specifies what feedback the human could have given. Our core contribution is to categorize choice set misspecification into several formally and empirically distinguishable ``classes'', and find that different types have significantly different effects on performance. As we might expect, misspecification is usually harmful; in the most extreme case the choice set is so misspecified that the robot believes the human feedback was the \emph{worst} possible feedback for the true reward, and so updates strongly towards the \emph{opposite} of the true reward. Surprisingly, we find that under other circumstances misspecification is \emph{provably neutral}: it neither helps nor hurts performance in expectation. 
    Crucially, these results suggest that not all misspecification is equivalently harmful to reward inference: we may be able to minimize negative impact by systematically erring toward particular misspecification classes defined in this work. Future work will explore this possibility.
    
    

\section{Reward Inference}\label{sec:inf}

    There are many ways that a human can provide feedback to a robot: demonstrations~\cite{ng2000algorithms,abbeel2004apprenticeship,ziebart2010modeling}, comparisons~\cite{sadigh2017active,Christiano2017}, natural language~\cite{goyal2019using}, corrections~\cite{bajcsy2017learning}, the state of the world~\cite{shah2019preferences}, proxy rewards~\cite{hadfield2017inverse,mindermann2018active}, etc. \citeauthor{jeon2020reward} propose a unifying formalism for reward inference to capture all of these possible feedback modalities, called reward-rational (implicit) choice (RRiC). Rather than study each feedback modality separately, we study misspecification in this general framework.
    
    RRiC consists of two main components: the human's \emph{choice set}, which corresponds to what the human \emph{could have done}, and the \emph{grounding function}, which converts choices into (distributions over) trajectories so that rewards can be computed.
    
    For example, in the case of learning from comparisons, the human chooses which out of two trajectories is better. Thus, the human's choice set is simply the set of trajectories they are comparing, and the grounding function is the identity. A more complex example is learning from the state of the world, in which the robot is deployed in an environment in which a human has already acted for $T$ timesteps, and must infer the human's preferences from the current world state. In this case, the robot can interpret the human as choosing between different possible states. Thus, the choice set is the set of possible states that the human could reach in $T$ timesteps, and the grounding function maps each such state to the set of trajectories that could have produced it.
    
    Let $\xi$ denote a trajectory and $\Xi$ denote the set of all possible trajectories. Given a choice set $C$ for the human and grounding function $\psi : C \rightarrow (\Xi \rightarrow [0, 1])$, \citeauthor{jeon2020reward} define a procedure for reward learning. They assume that the human is \emph{Boltzmann-rational} with rationality parameter $\beta$, so that the probability of choosing any particular feedback is given by:
    
    \begin{equation}
        \mathbb{P}(c \mid \theta, C) = \frac{exp(\beta \cdot \mathbb{E}_{\xi \sim \psi(c)}[r_\theta(\xi)])}{\sum_{c'\in C} exp(\beta \cdot \mathbb{E}_{\xi \sim \psi(c')}[r_\theta(\xi)])}
    \end{equation}
    
    From the robot's perspective, every piece of feedback $c$ is an observation about the true reward parameterization $\theta^*$, so the robot can use Bayesian inference to infer a posterior over $\theta$. Given a prior over reward parameters $\mathbb{P}(\theta)$, the RRiC inference procedure is defined as:
    
    \begin{equation}\label{eq:rrc}
        \mathbb{P}(\theta\mid c, C) \propto \frac{exp(\beta \cdot \mathbb{E}_{\xi \sim \psi(c)}[r_\theta(\xi)]}{\sum_{c'\in C} exp(\beta \cdot \mathbb{E}_{\xi \sim \psi(c')}[r_\theta(\xi)])} \cdot \mathbb{P}(\theta)
    \end{equation}
    
    Since we care about misspecification of the choice set $C$, we focus on learning from demonstrations, where we restrict the set of trajectories that the expert can demonstrate. This enables us to have a rich choice set, while allowing for a simple grounding function (the identity). In future work, we aim to test choice set misspecification with other feedback modalities as well.
    
\section{Choice Set Misspecification}\label{sec:class}

    For many common forms of feedback, including demonstrations and proxy rewards, the RRiC choice set is \textit{implicit}. The robot knows which element of feedback the human provided (ex. which demonstration they performed), but must assume which elements of feedback the human \textit{could have} provided based on their model of the human. However, this assumption could easily be incorrect -- the robot may assume that the human has capabilities that they do not, or may fail to account for cognitive biases that blind the human to particular feedback options, such as the human bias towards the most visually attention-grabbing choice in Fig~\ref{fig:front_fig}.
    
    To model such effects, we assume that the human selects feedback $c \in \Chu$ according to $\mathbb{P}(c\mid\theta, \Chu)$, while the robot updates their belief assuming a \emph{different} choice set $\Cro$ to get $\mathbb{P}(\theta\mid c, \Cro)$. Note that $\Cro$ is the robot's assumption about what the human's choice set is -- this is distinct from the robot's action space. When $\Chu \neq \Cro$, we get \textit{choice set misspecification}.
    
    It is easy to detect such misspecification when the human chooses feedback $c \notin \Cr$. In this case, the robot observes a choice that it believes to be impossible, which should certainly be grounds for reverting to some safe baseline policy. So, we only consider the case where the human's choice $c$ is also present in $\Cr$ (which also requires $\Ch$ and $\Cr$ to have at least one element in common).
    
    Within these constraints, we propose a classification of types of choice set misspecification in Table~\ref{tab:class}. On the vertical axis, misspecification is classified according to the location of the optimal element of feedback $c^* = {\mathrm{argmax}}_{c \in \Cr\cup\Ch}\;\mathbb{E}_{\xi \sim \psi(c)}[r_{\theta^*}(\xi))]$. If $c^*$ is \textit{available} to the human (in $\Ch$), then the class code begins with \texttt{A}. We only consider the case where $c^*$ is also in $\Cr$: the case where it is in $\Ch$ but not $\Cr$ is uninteresting, as the robot would observe the ``impossible'' event of the human choosing $c^*$, which immediately demonstrates misspecification at which point the robot should revert to some safe baseline policy. If $c^* \notin \Ch$, then we must have $c^* \in \Cr$ (since it was chosen from $\Ch \cup \Cr$), and the class code begins with \texttt{B}. On the horizontal axis, misspecification is classified according to the relationship between $\Cr$ and $\Ch$. $\Cr$ may be a subset (code \texttt{1}), superset (code \texttt{2}), or intersecting class (code \texttt{3}) of $\Ch$. For example, class \texttt{A1} describes the case in which the robot's choice set is a subset of the human's (perhaps because the human is more versatile), but both choice sets contain the optimal choice (perhaps because it is obvious).
    
    \begin{table}
        \centering
        \begin{tabular}{lllll}
                              &       & $\Cr\subset\Ch$ & $\Cr\supset\Ch$ & $\Cr\cap\Ch$ \\ \cline{3-5}
        \multirow{2}{*}{$c^*$} & \multicolumn{1}{l|}{$\in\Cr\cap\Ch$} & \texttt{A1}      & \texttt{A2}   & \texttt{A3}   \\
                              & \multicolumn{1}{l|}{$\in\Cr\backslash\Ch$}    &       & \texttt{B2}   & \texttt{B3}     \\        
        \end{tabular}
        \vspace{10pt}
        \caption{Choice set misspecification classification, where $\Cr$ is the robot's assumed choice set, $\Ch$ is the human's actual choice set, and $c^*$ is the optimal element from $\Cr\cup\Ch$. \texttt{B1} is omitted because if $\Cr \subset \Ch$, then $\Cr \backslash \Ch$ is empty and cannot contain $c^*$. }
        \label{tab:class}
        \title{Title}
    \end{table}

\section{Experiments}\label{sec:exp}

    To determine the effects of misspecification class, we artificially generated $\Cr$ and $\Ch$ with the properties of each particular class, simulated human feedback, ran RRiC reward inference, and then evaluated the robot's resulting belief distribution and optimal policy.

        \begin{figure}
            \centering
            \includegraphics[width=0.95\linewidth]{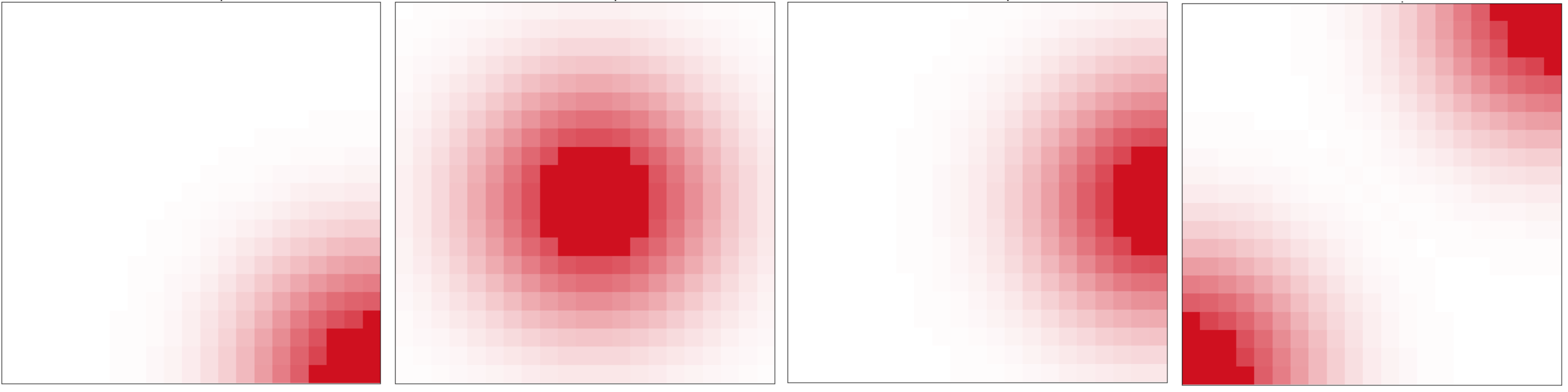}
            \vspace{5pt}
            \caption{The set of four gridworlds used in randomized experiments, with the lava feature marked in red.}
            \label{fig:envs}
        \end{figure}
    
    \subsection{Experimental Setup}
    
        \paragraph{Environment}
    
            To isolate the effects of misspecification and allow for computationally tractable Bayesian inference, we ran experiments in toy environments. We ran the randomized experiments in the four $20\times20$ gridworlds shown in Fig~\ref{fig:envs}. Each square in environment $x$ is a state $s^x = \{lava, goal\}$. $lava\in[0,1]$ is a continuous feature, while $goal\in\{0,1\}$ is a binary feature set to $1$ in the lower-right square of each grid and $0$ everywhere else. The true reward function $r_{\theta^*}$ is a linear combination of these features and a constant stay-alive cost incurred at each timestep, parameterized by $\theta = (w_{lava}, w_{goal}, w_{alive})$. Each episode begins with the robot in the upper-left corner and ends once the robot reaches the goal state or episode length reaches the horizon of 35 timesteps. Robot actions $A_R$ move the robot one square in a cardinal or diagonal direction, with actions that would move the robot off of the grid causing it to remain in place. The transition function $T$ is deterministic. Environment $x$ defines an MDP $M^x = \langle S^x, A_R, T, r_{\theta^*}\rangle$.
        
        \paragraph{Inference}
        
            While the RRiC framework enables inference from many different types of feedback, we use demonstration feedback here because demonstrations have an implicit choice set and straightforward deterministic grounding. Only the human knows their true reward function parameterization $\theta^*$. The robot begins with a uniform prior distribution over reward parameters $\mathbb{P}(\theta)$ in which $w_{lava}$ and $w_{alive}$ vary, but $w_{goal}$ always $=2.0$. $\mathbb{P}(\theta)$ contains $\theta^*$. RRiC inference proceeds as follows for each choice set tuple $\langle\Cr,\Ch\rangle$ and environment $x$. First, the simulated human selects the best demonstration from their choice set with respect to the true reward $c_H = {\mathrm{argmax}}_{c \in \Ch}\;\mathbb{E}_{\xi \sim \psi(c)}[r_{\theta^*}(\xi))]$. Then, the simulated robot uses Eq.~\ref{eq:rrc} to infer a ``correct'' distribution over reward parameterizations $B_H(\theta) \triangleq \mathbb{P}(\theta \mid c, \Ch)$ using the true human choice set, and a ``misspecified'' distribution $B_R(\theta) \triangleq \mathbb{P}(\theta \mid c, \Cr)$ using the misspecified human choice set. In order to evaluate the effects of each distribution on robot behavior, we define new MDPs $M^x_H = \langle S^x, A_R, T, r_{\mathbb{E}[B_H(\theta)]}\rangle$ and $M^x_R = \langle S^x, A_R, T, r_{\mathbb{E}[B_R(\theta)]}\rangle$ for each environment, solve them using value iteration, and then evaluate the rollouts of the resulting deterministic policies according to the true reward function $r_{\theta^*}$.
    
    \subsection{Randomized Choice Sets}
    
        We ran experiments with randomized choice set selection for each misspecification class to evaluate the effects of class on entropy change and regret.

        \paragraph{Conditions}
        
            The experimental conditions are the classes of choice set misspecification in Table~\ref{tab:class}: \texttt{A1}, \texttt{A2}, \texttt{A3}, \texttt{B2} and \texttt{B3}. We tested each misspecification class on each environment, then averaged across environments to evaluate each class. For each environment $x$, we first generated a master set $C_M^x$ of all demonstrations that are optimal w.r.t. at least one reward parameterization $\theta$. For each experimental class, we randomly generated 6 valid $\langle\Cr,\,\Ch\rangle$ tuples, with $\Cr,\Ch \subseteq C_M^x$. Duplicate tuples, or tuples in which $c_H \notin \Cr$, were not considered.
        
        \paragraph{Measures}
        
            There are two key experimental measures: \textit{entropy change} and \textit{regret}. Entropy change is the difference in entropy between the correct distribution $B_H$, and the misspecified distribution $B_R$. That is, $\Delta H = H(B_H) - H(B_R)$. If entropy change is positive, then misspecification induces overconfidence, and if it is negative, then misspecification induces underconfidence. 
            
            Regret is the difference in return between the optimal solution to $M^x_H$, with the correctly-inferred reward parameterization, and the optimal solution to $M^x_R$, with the incorrectly-inferred parameterization, averaged across all 4 environments. If $\xi^{*x}_H$ is an optimal trajectory in $M^x_H$ and $\xi^{*x}_R$ is an optimal trajectory in $M^x_R$, then $regret = \frac{1}{4}\sum^3_{x=0}[r_{\theta^*}(\xi^{*x}_H) - r_{\theta^*}(\xi^{*x}_R)]$. Note that we are measuring regret relative to the optimal action under the \emph{correctly specified belief}, rather than optimal action under the true reward. As a result, it is possible for regret to be \emph{negative}, e.g. if the misspecification makes the robot become more confident in the true reward than it would be under correct specification, and so execute a better policy.
            
    \subsection{Biased Choice Sets}\label{sec:bias_exp}
    
        We also ran an experiment in a fifth gridworld where we select the human choice set with a realistic human bias to illustrate how choice set misspecification may arise in practice. In this experiment the human only considers demonstrations that end at the goal state because, to humans, the word ``goal'' can be synonymous with ``end'' (Fig~\ref{fig:bias_expert_cset}). However, to the robot, the goal is merely one of multiple features in the environment. The robot has no reason to privilege it over the other features, so the robot considers every demonstration that is optimal w.r.t some possible reward parameterization (Fig~\ref{fig:bias_agent_cset}). The trajectory that only the robot considers is marked in blue. We ran RRiC inference using this $\tuple{\Cr}{\Ch}$ and evaluated the results using the same measures described above.
        
        \begin{figure}
            \centering
            \begin{subfigure}[b]{0.23\textwidth}
                \centering
                \includegraphics[width=0.95\textwidth]{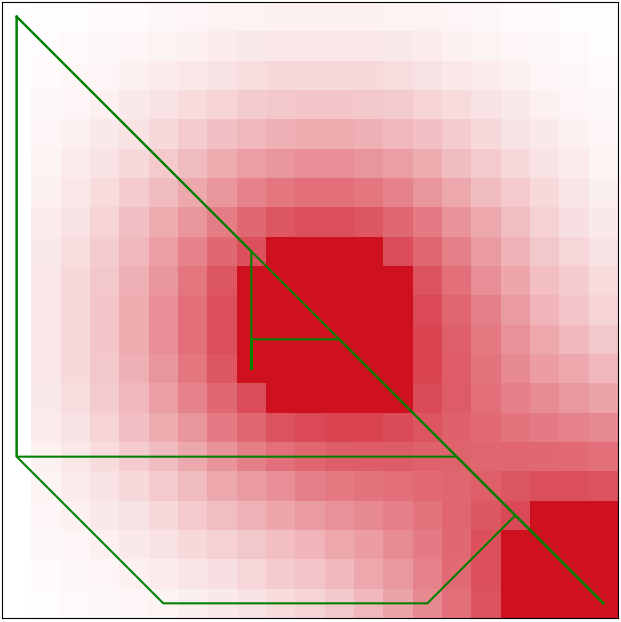} 
                \caption{$\Ch$}
                \label{fig:bias_expert_cset}
            \end{subfigure}
            \hspace{30pt}
            \begin{subfigure}[b]{0.23\textwidth}
                \centering
                \includegraphics[width=0.95\textwidth]{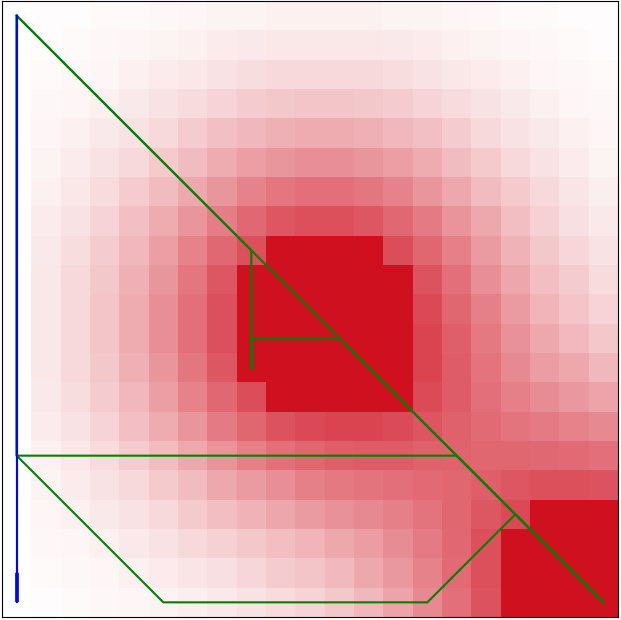} 
                \caption{$\Cr$}
                \label{fig:bias_agent_cset}
            \end{subfigure}
            \caption{Human and robot choice sets with a human goal bias. Because the human only considers trajectories that terminate at the goal, they don't consider the blue trajectory in $\Cr$.}
        \end{figure}
        
    
\section{Results}\label{sec:results}

    We summarize the aggregated measures, discuss the realistic human bias result, then examine two interesting results: symmetry between classes \texttt{A1} and \texttt{A2} and high regret in class \texttt{B3}.

    \subsection{Aggregate Measures in Randomized Experiments}
    
        \begin{figure}
            \centering
            \includegraphics[width=.6\linewidth]{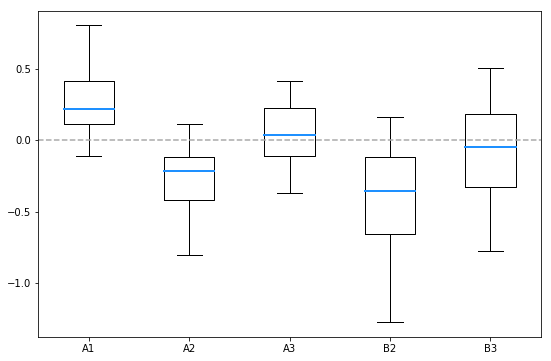}
            \caption{Entropy Change (N=24). The box is the IQR, the whiskers are the range, and the blue line is the median. There are no outliers.}
            \label{fig:ent}
        \end{figure}
        
        \begin{figure}
                \centering
                \includegraphics[width=.6\linewidth]{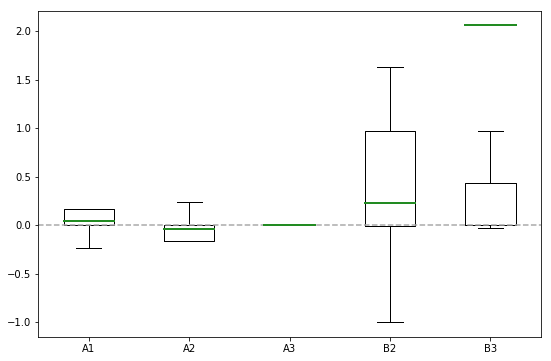}
                \caption{Regret (N=24). The box is the IQR, the whiskers are the most distant points within 1.5*the IQR, and the green line is the mean. Multiple outliers are omitted.}
                \label{fig:reg}
        \end{figure}
    
        \paragraph{Entropy Change}
        
            Entropy change varied significantly across misspecification class. As shown in Fig~\ref{fig:ent}, the interquartile ranges (IQRs) of classes \texttt{A1} and \texttt{A3} did not overlap with the IQRs of \texttt{A2} and \texttt{B2}. Moreover, \texttt{A1} and \texttt{A3} had positive medians, suggesting a tendency toward overconfidence, while \texttt{A2} and \texttt{B2} had negative medians, suggesting a tendency toward underconfidence. \texttt{B3} was less distinctive, with an IQR that overlapped with that of all other classes. Notably, the distributions over entropy change of classes \texttt{A1} and \texttt{A2} are precisely symmetric about 0.
            
        \paragraph{Regret}
            
            Regret also varied as a function of misspecification class. Each class had a median regret of 0, suggesting that misspecification commonly did not induce a large enough shift in belief for the robot to learn a different optimal policy. However the mean regret, plotted as green lines in Fig~\ref{fig:reg}, did vary markedly across classes. Regret was sometimes so high in class \texttt{B3} that outliers skewed the mean regret beyond of the whiskers of the boxplot. Again, classes \texttt{A1} and \texttt{A2} are precisely symmetric. We discuss this symmetry in Section~\ref{sec:sym}, then discuss the poor performance of \texttt{B3} in Section~\ref{sec:worst}.
   
    \subsection{Effects of Biased Choice Sets}\label{sec:bias_res}
    
        The human bias of only considering demonstrations that terminate at the goal leads to very poor inference in this environment. Because the human does not consider the blue demonstration from Fig~\ref{fig:bias_agent_cset}, which avoids the lava altogether, they are forced to provide the demonstration in Fig~\ref{fig:bias_feedback}, which terminates at the goal but is long and encounters lava. As a result, the robot infers the very incorrect belief distribution in Fig~\ref{fig:bias_misspec_belief}. Not only is this distribution underconfident (entropy change = $-0.614$), but it also induces poor performance (regret = $0.666$). This result shows that we can see an outsized negative impact on robot reward inference with a small incorrect assumption that the human considered and rejected demonstrations that don't terminate at the goal.
        
        \begin{figure}
            \centering
            \begin{subfigure}{0.23\textwidth}
                \centering
                \includegraphics[width=0.95\textwidth]{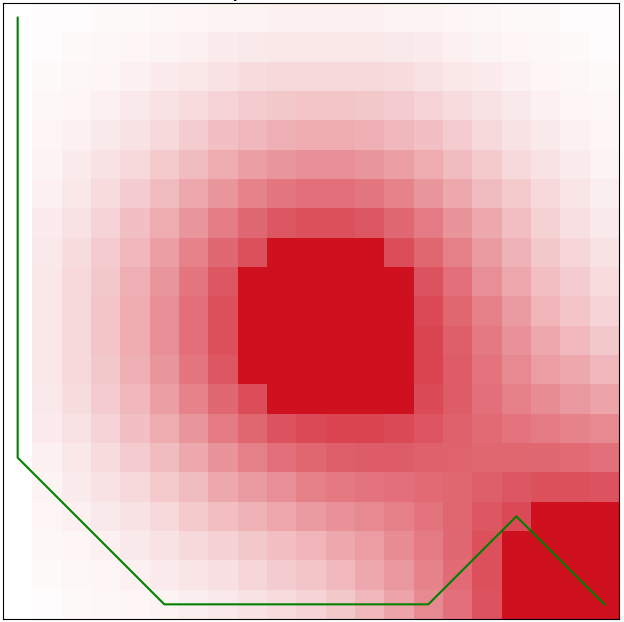}
                \caption{feedback $c_H$}
                \label{fig:bias_feedback}
            \end{subfigure}
            \hspace{30pt}
            \begin{subfigure}{0.23\textwidth}
                \centering
                \includegraphics[width=0.95\textwidth]{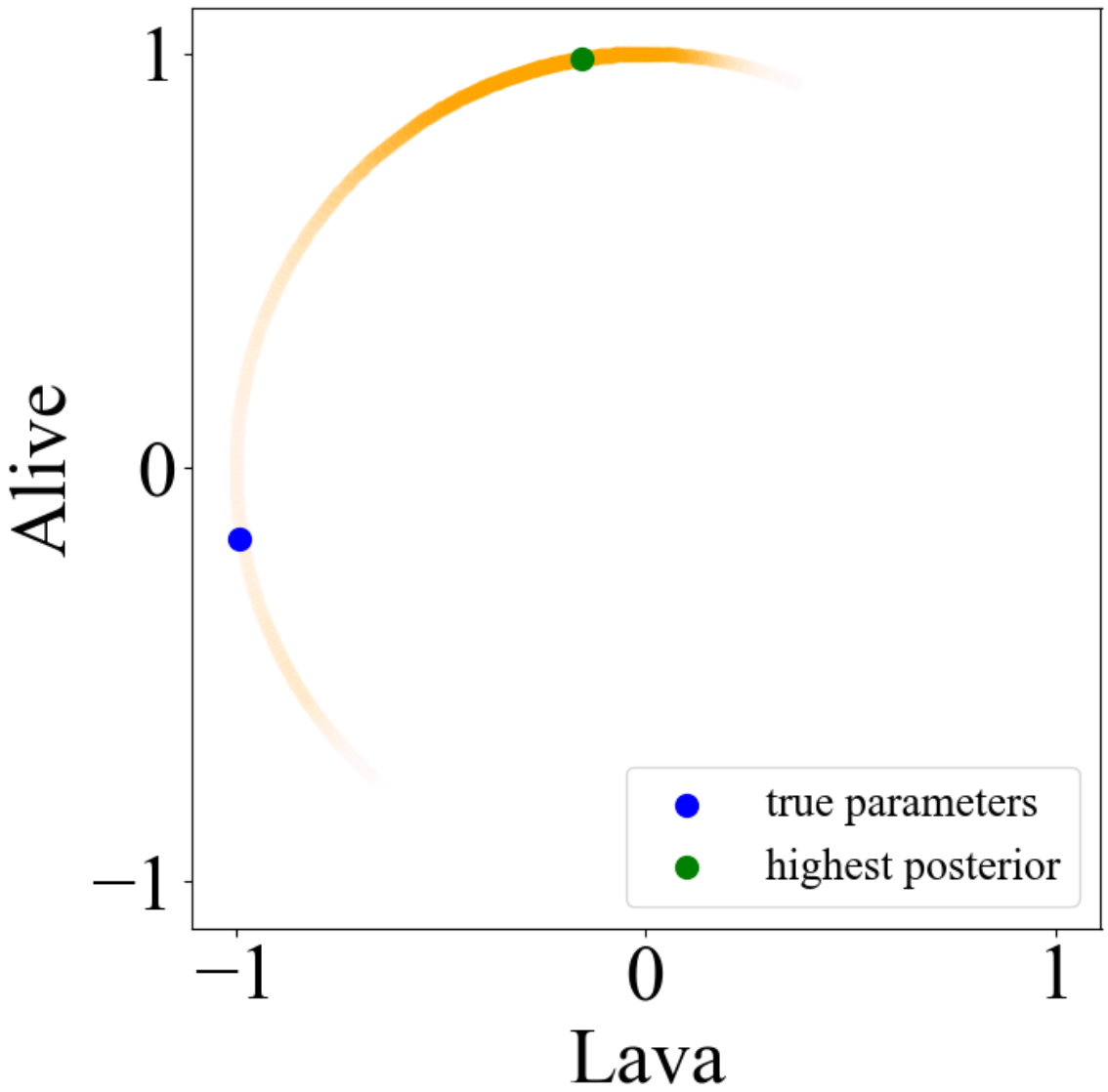}
                \caption{$\mathbb{P}(\theta\mid c_H,\Cr)$}
                \label{fig:bias_misspec_belief}
            \end{subfigure}
            \caption{Human feedback and the resulting misspecified robot belief with a human goal bias. Because the feedback that the biased human provides is poor, the robot learns a very incorrect distribution over rewards.}
        \end{figure} 
        
    \subsection{Symmetry}\label{sec:sym}
        Intuitively, misspecification should lead to worse performance in expectation. Surprisingly, when we combine misspecification classes \texttt{A1} and \texttt{A2}, their impact on entropy change and regret is actually neutral. The key to this is their symmetry -- if we switch the contents of $\Cro$ and $\Chu$ in an instance of class \texttt{A1} misspecification, we get an instance of class \texttt{A2} with exactly the opposite performance characteristics. Thus, if a pair in \texttt{A1} is harmful, then the analogous pair in \texttt{A2} must be helpful, meaning that it is \emph{better for performance than having the correct belief about the human's choice set}. We show below that this is always the case under certain symmetry conditions that apply to \texttt{A1} and \texttt{A2}.
        
        Assume that there is a master choice set $C_M$ containing all possible elements of feedback for MDP $M$, and that choice sets are sampled from a symmetric distribution over pairs of subsets $D : 2^{C_M} \times 2^{C_M} \rightarrow [0,1]$ with $D(C_x, C_y) = D(C_y, C_x)$ (where $2^{C_M}$ is the set of subsets of $C_M$). Let $ER(r_\theta, M)$ be the expected return from maximizing the reward function $r_\theta$ in $M$. A reward parameterization is chosen from a shared prior $\mathbb{P}(\theta)$ and $\Ch, \Cr$ are sampled from $D$. The human chooses the optimal element of feedback in their choice set $c_{\Ch} = {\mathrm{argmax}}_{c \in \Ch}\;\mathbb{E}_{\xi \sim \psi(c)}[r_{\theta^*}(\xi))]$.
    
        \begin{theorem}\label{theorem:sym}
            Let $M$ and $D$ be defined as above. Assume that $\forall C_x, C_y \sim D$, we have $c_{C_x} = c_{C_y}$; that is, the human would pick the same feedback regardless of which choice set she sees. If the robot follows RRiC inference according to Eq.~\ref{eq:rrc} and acts to maximize expected reward under the inferred belief, then:
            
            $$\mathop{\mathbb{E}}_{\Ch, \Cr \sim D}Regret(\Ch, \Cr) = 0$$
        \end{theorem}
        
        \begin{proof}
            Define $R(C_x, c)$ to be the return achieved when the robot follows RRiC inference with choice set $C_x$ and feedback $c$, then acts to maximize $r_{\mathbb{E}[B_x(\theta)]}$, keeping $\beta$ fixed. Since the human's choice is symmetric across $D$, for any $C_x, C_y \sim D$, regret is \emph{anti-}symmetric:
            \begin{align*}
                Regret(C_x, C_y)
                &= R(C_x, c_{C_x}) - R(C_y, c_{C_x}) \\
                &= R(C_x, c_{C_y}) - R(C_y, c_{C_y}) \\
                &= - Regret(C_y, C_x)
            \end{align*}
            Since $D$ is symmetric, $\tuple{C_x}{C_y}$ is as likely as $\tuple{C_y}{C_x}$. Combined with the anti-symmetry of regret, this implies that the expected regret must be zero:
            \begin{align*}
                &\ \ \mathop{\mathbb{E}}_{C_x, C_y \sim D}[Regret(C_x, C_y)] \\
                &= \frac{1}{2}\mathop{\mathbb{E}}_{C_x, C_y}[Regret(C_x, C_y)] + \frac{1}{2}\mathop{\mathbb{E}}_{C_x, C_y}[Regret(C_y, C_x)]\\
                &= \frac{1}{2}\mathop{\mathbb{E}}_{C_x, C_y}[Regret(C_x, C_y)] - \frac{1}{2}\mathop{\mathbb{E}}_{C_x, C_y}[Regret(C_x, C_y)]\\
                &= 0 
            \end{align*}
        \end{proof}
        
        An analogous proof would work for any anti-symmetric measure (including entropy change).
        
        \begin{table}
            \begin{center}
                \begin{tabular}[h]{ ccccc } \\\toprule
                    Class & Mean    & Std    & Q1      & Q3      \\ \midrule\midrule
                    A1    & 0.256  & 0.2265 & 0.1153  & 0.4153  \\ \midrule
                    A2    & -0.256 & 0.2265 & -0.4153 & -0.1153 \\ \bottomrule
                \end{tabular}
                \vspace{10pt}
                \caption{Entropy change is symmetric across classes \texttt{A1} and \texttt{A2}.}
                \label{tab:ent}
            \end{center}
        \end{table}
        \begin{table}
            \begin{center}
                \begin{tabular}[h]{ ccccc } \\\toprule
                    Class & Mean    & Std    & Q1      & Q3      \\ \midrule\midrule
                    A1    & 0.04  & 0.4906 & 0.1664  & 0.0  \\ \midrule
                    A2    & -0.04 & 0.4906 & 0.0 & -0.1664 \\ \bottomrule
                \end{tabular}
                \vspace{10pt}
                \caption{Regret is symmetric across classes \texttt{A1} and \texttt{A2}.}
                \label{tab:reg}
            \end{center}
        \end{table}
    
    \subsection{Worst Case}\label{sec:worst}
    
        \begin{table}
            \centering
            \begin{tabular}{ ccccc }     \\\toprule
                Class           & Mean              & Std               & Max                   & Min               \\ \midrule\midrule
                A3              & -0.001            & 0.5964            & 1.1689                & -1.1058           \\ \midrule
                B2              & 0.228             & 0.6395            & 1.6358                & -0.9973           \\ \midrule
                \textbf{B3}    & \textbf{2.059}     & \textbf{6.3767}   & \textbf{24.7252}      & \textbf{-0.9973}  \\ \bottomrule
            \end{tabular}
            \vspace{10pt}
            \caption{Regret comparison showing that class \texttt{B3} has much higher regret than neighboring classes.}
            \label{tab:regB3}
        \end{table} 
        
        As shown in Table~\ref{tab:regB3}, class \texttt{B3} misspecification can induce regret an order of magnitude worse than the maximum regret induced by classes \texttt{A3} and \texttt{B2}, which each differ from \texttt{B3} along a single axis. This is because the worst case inference occurs in RRiC when the human feedback $c_H$ is the worst element of $\Cr$, and this is only possible in class \texttt{B3}. In class \texttt{B2}, $C_R$ contains all of $C_H$, so as long as $|C_H|>1$, $C_R$ must contain at least one element worse than $C_H$. In class \texttt{A3}, $c_H = c^*$, so $C_R$ cannot contain any elements better than $c_H$. However, in class \texttt{B3}, $\Cr$ need not contain any elements worse than $c_H$, in which case the robot updates its belief in the opposite direction from the ground truth.
        
        For example, consider the sample human choice set in Fig~\ref{fig:expert_cset}. Both trajectories are particularly poor, but the human chooses the demonstration $c_H$ in Fig~\ref{fig:feedback} because it encounters slightly less lava and so has a marginally higher reward. Fig~\ref{fig:agent_cset_B2} shows a potential corresponding robot choice set $C_{R2}$ from \texttt{B2}, containing both trajectories from the human choice set as well as a few others. Fig~\ref{fig:belief_B2_large} shows $\mathbb{P}(\theta\mid c_H, C_{R2})$. The axes represent the weights on the $lava$ and $alive$ features and the space of possible parameterizations lies on the circle where $w_{lava} + w_{alive} = 1$. The opacity of the gold line is proportional to the weight that $\mathbb{P}(\theta)$ places on each parameter combination. The true reward has $w_{lava},w_{alive} < 0$, whereas the peak of this distribution has $w_{lava} < 0$, but $w_{alive} > 0$. This is because $C_{R2}$ contains shorter trajectories that encounter the same amount of lava, and so the robot infers that $c_H$ must be preferred in large part due to its length.
        
        Fig~\ref{fig:agent_cset_B3} shows an example robot choice set $C_{R3}$ from \texttt{B3}, and Fig~\ref{fig:belief_B3} shows the inferred $\mathbb{P}(\theta\mid c_H, C_{R3})$. Note that the peak of this distribution has $w_{lava},w_{alive}>0$. Since $c_H$ is the longest and the highest-lava trajectory in $C_{R3}$, \textit{and} alternative shorter and lower-lava trajectories exist in $C_{R3}$, the robot infers that the human is attempting to maximize both trajectory length and lava encountered: the \emph{opposite} of the truth. Unsurprisingly, maximizing expected reward for this belief leads to high regret. The key difference between \texttt{B2} and \texttt{B3} is that $c_H$ is the lowest-reward element in $C_{R3}$, resulting in the robot updating directly away from the true reward.
        
        \begin{figure}
            \centering
            \begin{subfigure}{0.23\textwidth}
                \centering
                \includegraphics[width=0.95\textwidth]{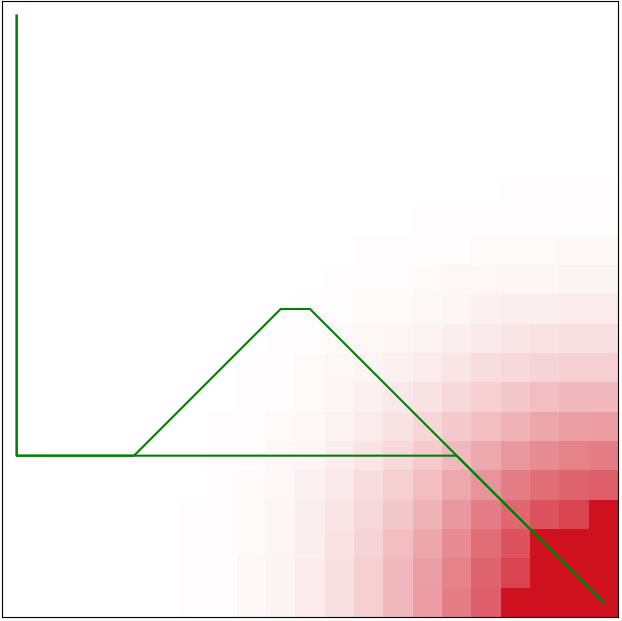} 
                \caption{$\Ch$}
                \label{fig:expert_cset}
            \end{subfigure}
            \hspace{30pt}
            \begin{subfigure}{0.23\textwidth}
                \centering
                \includegraphics[width=0.95\textwidth]{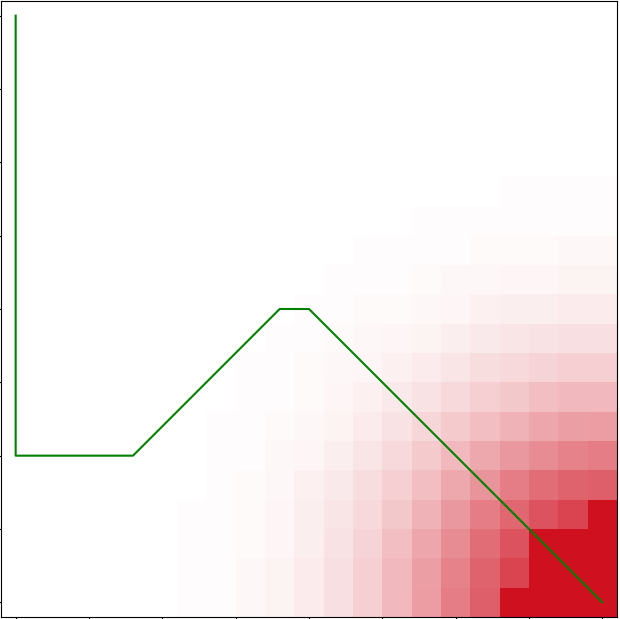} 
                \caption{$c_H$}
                \label{fig:feedback}
            \end{subfigure}
            \caption{Example human choice set and corresponding feedback.}
        \end{figure}
        
        \begin{figure}
            \centering
            \begin{subfigure}{0.23\textwidth}
                \centering
                \includegraphics[width=0.95\textwidth]{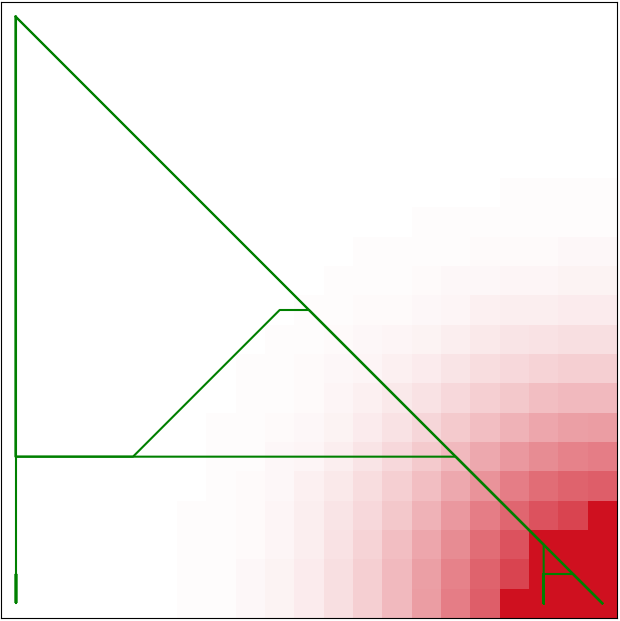}
                \caption{$C_{R2}$}
                \label{fig:agent_cset_B2}
            \end{subfigure}
            \hspace{30pt}
            \begin{subfigure}{0.23\textwidth}
                \centering
                \includegraphics[width=0.95\textwidth]{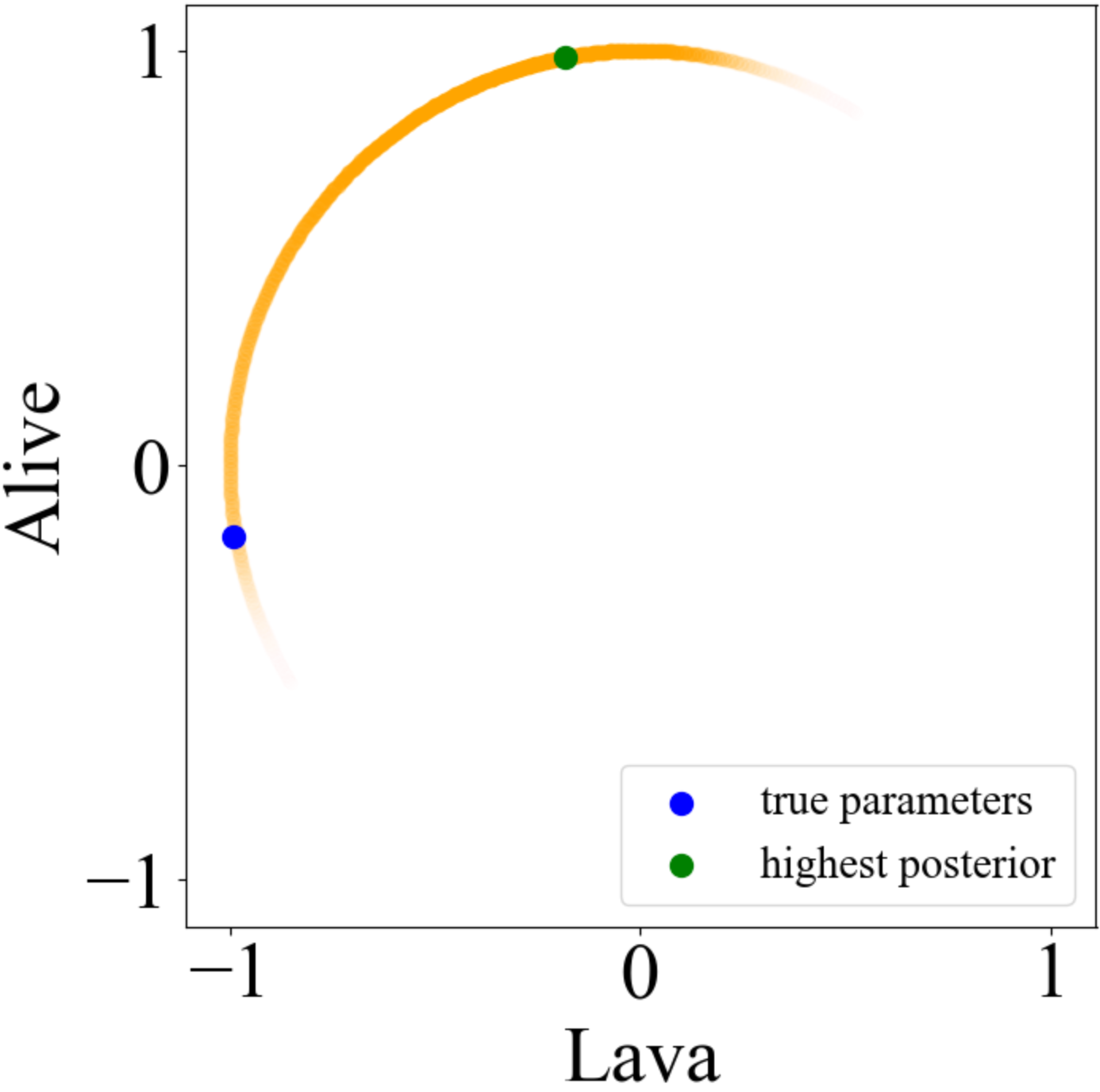}
                \caption{$\mathbb{P}(\theta\mid c_H,C_{R2})$}
                \label{fig:belief_B2_large}
            \end{subfigure}
            \caption{Robot choice set and resulting misspecified belief in \texttt{B2}.}
        \end{figure}   
        
        \begin{figure}
            \centering
            \begin{subfigure}{0.23\textwidth}
                \centering
                \includegraphics[width=0.95\textwidth]{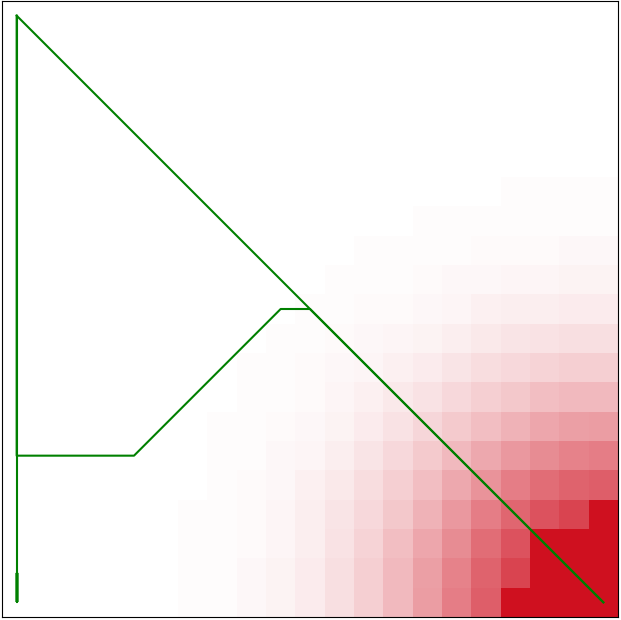}
                \caption{$C_{R3}$}
                \label{fig:agent_cset_B3}
            \end{subfigure}
            \hspace{30pt}
            \begin{subfigure}{0.23\textwidth}
                \centering
                \includegraphics[width=0.95\textwidth]{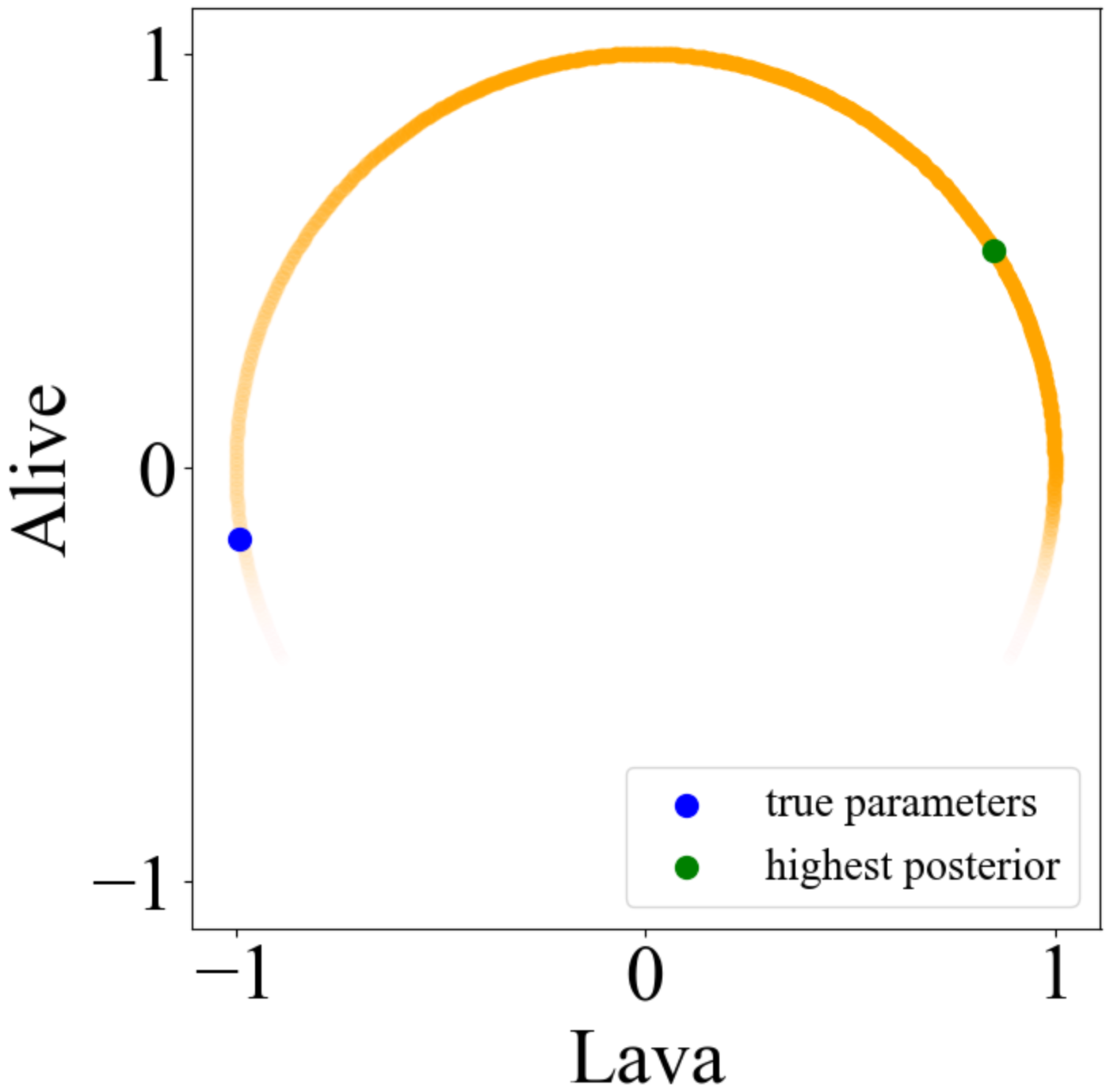}
                \caption{$\mathbb{P}(\theta\mid c_H,C_{R3})$}
                \label{fig:belief_B3}
            \end{subfigure}
            \caption{Robot choice set and resulting misspecified belief in \texttt{B3}.}
        \end{figure}  
        
\section{Discussion}\label{sec:dis}
    
    \paragraph{Summary}
    
        In this work, we highlighted the problem of choice set misspecification in generalized reward inference, where a human gives feedback selected from choice set $\Chu$ but the robot assumes that the human was choosing from choice set $\Cro$. As expected, such misspecification on average induces suboptimal behavior resulting in regret. However, a different story emerged once we distinguished between misspecification classes. We defined five distinct classes varying along two axes: the relationship between $\Chu$ and $\Cro$ and the location of the optimal element of feedback $c^*$. We empirically showed that different classes lead to different types of error, with some classes leading to overconfidence, some to underconfidence, and one to particularly high regret. Surprisingly, under certain conditions the expected regret under choice set misspecification is actually 0, meaning that in expectation, misspecification does not hurt in these situations.
        
    \paragraph{Implications}
    
        There is wide variance across the different types of choice-set misspecification: some may have particularly detrimental effects, and others may not be harmful at all. This suggests strategies for designing robot choice sets to minimize the impact of misspecification. For example, we find that regret tends to be \emph{negative} (that is, misspecification is \emph{helpful}) when the optimal element of feedback is in both $\Cro$ and $\Chu$ and $\Cro \supset \Chu$ (class \texttt{A2}). Similarly, worst-case inference occurs when the optimal element of feedback is in $\Cro$ only, and $\Chu$ contains elements that are not in $\Cro$ (class \texttt{B3}). This suggests that erring on the side of specifying a large $\Cro$, which makes \texttt{A2} more likely and \texttt{B3} less, may lead to more benign misspecification. Moreover, it may be possible to design protocols for the robot to identify unrealistic choice set-feedback combinations and verify its choice set with the human, reducing the likelihood of misspecification in the first place. We plan to investigate this in future work.
    
    \paragraph{Limitations and future work.}
    
        In this paper, we primarily sampled choice sets randomly from the master choice set of all possibly optimal demonstrations. However, this is not a realistic model. In future work, we plan to select human choice sets based on actual human biases to improve ecological validity. We also plan to test this classification and our resulting conclusions in more complex and realistic environments. Eventually, we plan to work on active learning protocols that allow the robot to identify when its choice set is misspecified and alter its beliefs accordingly.

\section*{Acknowledgements}

    We thank colleagues at the Center for Human-Compatible AI for discussion and feedback. This work was partially supported by an ONR YIP.
        
\bibliographystyle{named}
\bibliography{refs}
    
\end{document}